\documentclass[12pt]{article}

\usepackage{psfrag}
\usepackage{color}
\usepackage{algorithm}
\usepackage{algorithmic}
\usepackage{amsfonts}
\usepackage{amsmath}
\usepackage{mathrsfs}
\usepackage{amssymb}
\usepackage{wrapfig}
\usepackage{graphicx}
\usepackage{psfrag,subfigure} 

\newcommand{\bm}[1]{\boldsymbol{#1}}
\newcommand{\bY}{\mathbf{Y}}
\newcommand{\by}{\mathbf{y}}
\newcommand{\bX}{\mathbf{X}}
\newcommand{\bx}{\mathbf{x}}

\newcommand{\bB}{\mathbf{B}}
\newcommand{\bE}{\mathbf{E}}
\newcommand{\bg}{\mathbf{g}}
\newcommand{\bh}{\mathbf{h}}
\newcommand{\bG}{\mathbf{G}}
\newcommand{\bH}{\mathbf{H}}

\def\bSig\mathbf{\Sigma}

\newcommand{\diag}{\mathop{\mathrm{diag}}}

\newcommand\Ldots{\mathop{\lower.01ex\hbox{$\hdots$}}}

\newtheorem{theorem}{Theorem}

\newtheorem{proposition}[theorem]{Proposition}

\newenvironment{proof}[1][Proof]{\begin{trivlist}
\item[\hskip \labelsep {\bfseries #1}]}{\end{trivlist}}

\newenvironment{example}[1][Example]{\begin{trivlist}
\item[\hskip \labelsep {\bfseries #1}]}{\end{trivlist}}

\newcommand{\qed}{\hfill \ensuremath{\Box}}

\usepackage{JASA_manu}
\usepackage{hyperref}

\usepackage{color}
\def\bSig\mathbf{\Sigma}

\begin{document}

\title{Efficient Algorithm for Extremely Large\\
			 Multi-task Regression with Massive Structured Sparsity}
\author{Seunghak Lee$^{*}$  \\
Eric P. Xing$^{**}$ \\ 
School of Computer Science \\ 
Carnegie Mellon University, Pittsburgh, PA, U.S.A. \\ 
$^{*}$email: \texttt{seunghak@cs.cmu.edu}\\
$^{**}$email: \texttt{epxing@cs.cmu.edu} }

\maketitle

\begin{center}
\textbf{Abstract}
\end{center}
We develop a highly scalable optimization method called
``hierarchical group-thresholding'' 
for solving a multi-task regression
model with complex structured sparsity constraints on both input and output spaces. 
Despite the recent emergence of several efficient optimization algorithms for tackling complex  sparsity-inducing regularizers, true scalability in practical high-dimensional problems where a huge amount (e.g., millions) of sparsity patterns 
need to be enforced remains an open challenge, because all existing algorithms must deal with ALL such patterns exhaustively in every iteration, which is computationally prohibitive.    
Our proposed algorithm addresses the scalability problem 
by screening out multiple groups of coefficients simultaneously and systematically.
We employ a hierarchical tree representation of 
group constraints to accelerate the process of removing irrelevant constraints by taking advantage of the inclusion relationships between group sparsities, thereby avoiding dealing with all constraints in every optimization step, and necessitating optimization operation only on a small number of outstanding coefficients.
In our experiments, 
we demonstrate the efficiency of our method on simulation datasets, and in an application  
of detecting genetic variants associated with gene expression traits.

\vspace*{.3in}

\section{Introduction}

In this paper, we propose a very efficient
optimization 
technique for multi-task regression with structured sparsity.
We are interested in the optimization problem with the following general form: 
\begin{equation}
\label{equ:general}
\min_{\bB} \frac{1}{2}
\lVert \bY - \bB \bX  \rVert_F^2 
 + \lambda_1 \left| \bB \right| + \lambda_2 \Omega_{in}(\bB) + \lambda_3 \Omega_{out}(\bB)
\end{equation}
where $\bX \in \mathbb{R}^{J \times N}$ is the input data for  $J$ inputs and $N$ samples, 
$\bY \in \mathbb{R}^{K \times N}$ is the $K$ output data 
(equivalently $K$ tasks), and
$\bB \in \mathbb{R}^{K \times J}$ is the regression coefficient matrix.
Here $\Omega_{in}$ is an $\ell_1/\ell_2$ norm 
for inducing group sparsity among correlated inputs (grouping effects in the same rows of $\bB$) and
$\Omega_{out}$ is an $\ell_1/\ell_2$ norm 
for inducing group sparsity among correlated outputs (grouping effects in the same columns of $\bB$).
In this setting, it is possible that there exists 
overlap between/within input and output groups 
(i.e., a row group and a column group may intersect and hence overlap).
Note that this formulation subsumes popular special cases such as single task lasso, group lasso, etc..
However, throughout this paper, we use the  formulation in (\ref{equ:general}), 
as it explicitly presents a highly general regression problem, and
one can still use our algorithm for a single task regression problem
by setting $\lambda_3=0$ and $K=1$.

Unfortunately, problem (\ref{equ:general}) is non-trivial to optimize as
it poses two major challenges for large scale problems.
First, we need to be able to handle 
a large number of group sparsities efficiently.
For example, in eQTL mapping problems in bioinformatics, there exist a very large number of groups since
the number of input and output groups are 
proportional to $K$ (e.g., $2\times10^4$) and $J$ (e.g., $5\times10^5$), respectively.
Second, we need to deal with overlap of groups 
within 
and between
$\Omega_{in}$ and $\Omega_{out}$.
Note that a simple coordinate descent algorithm is not applicable
when $\Omega_{in}$ or $\Omega_{out}$ is non-separable.

The second challenge has been addressed by many 
optimization techniques including 
\cite{jenatton2009structured,jacob2009group,mairal2010network,chen2010efficient,yuanefficient,
mosci2010primal,qin2011structured,argyriou2011efficient,jenatton2011proximal,mairal2011convex, boyd2011distributed}.
For example, Jacob et al. \cite{jacob2009group} proposed to select
the union of overlapping groups as the support of sparse vectors. 
In their optimization procedure, input variables are duplicated to 
convert $\Omega_{in}$ with overlap into the norm with disjoint groups, and 
an optimization technique for group lasso \cite{meier2008group} is applied.
Jenatton et al. 
developed Structured-Lasso (SLasso) algorithm for sparsity-inducing norms
with overlapping groups \cite{jenatton2009structured}.
A smoothing proximal gradient method (SPG) \cite{chen2010efficient} is developed 
to efficiently deal with overlapping group lasso penalty and graph-guided fusion penalty.
Also, an efficient algorithm
based on alternating direction methods 
\cite{qin2011structured} was proposed for overlapping group lasso with both 
$\ell_1/\ell_2$ norm and $\ell_1/\ell_\infty$ norm.  
Recently, fast overlapping group lasso (FoGLasso) \cite{yuanefficient}
was proposed for fast optimization of overlapping group lasso problem
based on accelerated gradient descent method and a proximal operator.

However, the first  challenge is a scalability problem when there exist a very large
number of (overlapping) groups, and it has been relatively less studied in previous works. 
For example, the time complexity of 
smoothing proximal gradient method (SPG) \cite{chen2010efficient} 
is $O(\sum_{\bg_m \in \mathcal{G}} |\bg_m|)$, where 
$\mathcal{G}$ is a set of groups, and
a primal-dual algorithm for overlapping group lasso \cite{mosci2010primal}
has time complexity of $O(|\hat{\mathcal{G}}^3|)$, 
where $\hat{\mathcal{G}}$ is the set of active groups
(groups having non-zero elements).
At each iteration of SLasso algorithm \cite{jenatton2009structured},
there is an expensive matrix inversion operation, and
the inner loop of Picard-Nesterov method \cite{argyriou2011efficient} 
and FoGLasso \cite{yuanefficient} have the time complexity of $O(J|\mathcal{G}|)$.
As the number of groups in large-scale problems can be very large (e.g. $10^6$), the 
scalability of existing algorithms could be severely affected by a large number of groups.
Thus, there is an urgent need to develop an 
algorithm highly scalable to the number of groups, and
in this paper, we present a highly efficient algorithm 
given a very large number of (overlapping) groups.
Figure \ref{fig:speed_fullcomp} illustrates the efficiency of our method
in comparison to other competitors including FoGLasso, SPG, and SLasso.

We present a simple and efficient algorithm
called hierarchical group-thresholding method (HiGT) to
address the scalability problem for overlapping group lasso.
We use the following optimization strategy. 
First, we screen a large number of zero groups  simultaneously  
by testing the zero condition of multiple groups.
We further improved the speed of this step by employing
a tree data structure where nodes represent
the zero patterns encoded by $\Omega_{in}$ and $\Omega_{out}$ 
at different granularity, and edges indicate the inclusion relations among them.
Using the tree data structure, we can avoid checking a large number of zero groups.
Second, given a small number of nonzero groups of coefficients from the previous step,
we solve our problem using an efficient method for overlapping group lasso.  
We used FoGLasso for the second step.
It is also noteworthy that the accuracy of our screening step is not affected by 
the number of overlapping groups as it relies on exact optimality conditions
of zero groups. Unlike our method, a large number of overlapping groups 
can degrade the accuracy of some approximation approaches
(see Figure \ref{fig:speed_fullcomp}(b)).

In our experiments, we first evaluate the efficiency of
the first step (screening step). Then, 
we demonstrate the performance of our method in terms of 
the speed and the accuracy for the recovery of structured sparsity 
via simulation study, in comparison to three state-of-the-art methods. 
As an example of biological analysis, we report a novel and significant 
SNP pair identified by our method, and present discussions. 

\paragraph{Remark}
The problem (\ref{equ:general}) is originally motivated 
by expression quantitative trait loci (eQTLs) mapping in computational biology.
Here eQTLs refer to the genomic locations or single nucleotide polymorphisms 
(SNPs) associated with gene expressions.
In eQTL mapping problems, it is 
believed that many inputs (i.e., SNPs) impose small or medium effects on outputs (i.e., expression traits), and we usually have $J>>N$ ($J\sim 10^6, N\sim 10^3$) which exacerbate the noise to signal ratio. 
Thus, it is desirable to explore the groups of inputs  to 
increase effective signal strength (individual inputs have 
too small effects to be detected) for more accurate causal SNP identification.  
It is also desirable to perform multi-task learning by jointly considering multiple (possibly correlated) responses to decrease the sample size required for successful support recovery
\cite{negahban2011simultaneous} (the number of samples is too small to detect small signals). 
Thus, to take advantage of both input groups $\Omega_{in}$ 
and output groups $\Omega_{out}$ simultaneously,
we are interested in solving problem (\ref{equ:general}).

\paragraph{Notations}
Given a matrix $\mathbf{B} \in \mathbb{R}^{K \times J}$, we denote
the $k$-th row by $\bm{\beta}_k$, the $j$-th column
by $\bm{\beta}^j$, and the $(k,j)$ element 
by $\beta_k^j$. 
Given the set of groups $\mathcal{G}=\{{\bg_m}_1, \ldots, {\bg_m}_{|\mathcal{G}|}\}$ 
defined as a subset of the power set of $\{1, \ldots, J\}$,
$\bm{\beta}_k^{\bg_m}$ represents the row vector with
elements $\{\beta_k^j: j \in \bg_m, \; \bg_m \in \mathcal{G} \}$.
Similarly, for the set of groups $\mathcal{H}=\{{\bh}_1, \ldots, {\bh}_{|\mathcal{H}|}\}$ 
over $K$ rows of matrix $\mathbf{B}$, 
we denote by $\bm{\beta}_{\bh_o}^{j}$ the column vector with
elements $\{\beta_k^j: k \in \bh_o, \; \bh_o \in \mathcal{H} \}$. 
We also define the submatrix of $\mathbf{B}_{\bh_o}^{\bg_m}$ as 
a $|\bh_o| \times |\bg_m|$ matrix with elements $\{\beta_k^j: k \in \bh_o, \; j \in \bg_m, \;
\bh_o \in \mathcal{H}, \; \bg_m \in \mathcal{G} \}$.

\section{Multi-task Regression with Structured Sparsity}

We use a linear
model parametrized by unknown regression coefficients $\bB \in \mathbb{R}^{K \times J}$:
$\bY = \bB \bX + \bE$, 
where $\bE \in \mathbb{R}^{K \times N}$ 
is i.i.d. Gaussian noise with 
zero mean and the identity covariance matrix.
Throughout the paper, we assume that $x_j^{i}$s and $y_k^i$s 
are standardized,
and consider a model without an intercept.

Suppose that we are given a set of input groups $\mathcal{G}$ and
a set of output groups $\mathcal{H}$.  
We consider a multi-task regression model with  structured sparsity:
\begin{align}
\label{equ:reg5}
\min_{\bB} \frac{1}{2}
\lVert \bY - \bB \bX  \rVert_F^2 
+ \lambda_1   
\lVert  \diag{(\bm{w}^T \bB)} \rVert_1
+ \lambda_2
\sum_{k=1}^K \sum_{{\bg_m} \in \mathcal{G}} \rho_{t} \lVert\bm{\beta}_k^{\bg_m} \rVert_{2}
+ \lambda_3 
\sum_{j=1}^J \sum_{{\bh_o} \in \mathcal{H}} \nu_{o} \lVert \bm{\beta}_{\bh_o}^j \rVert_{2},
\end{align}
where $\bg_m \in \mathcal{G}$ is the $m$th group of inputs, 
$\bh_o \in \mathcal{H}$ is the $o$th group of outputs, 
$\lVert \bm{\beta}_k^{\bg_m} \rVert_{2} 
= 
\sqrt{\sum_{j \in \bg_m} (\beta_k^j)^{2} }$,
and 
$\lVert \bm{\beta}_{\bh_o}^k \rVert_{2} 
= 
\sqrt{\sum_{k \in \bh_o} (\beta_k^j)^{2} }$.
Here individual or groups of coefficients are differently penalized with
weights $\bm{w} \in \mathbb{R}^{K \times J}$, $\bm{\rho} \in \mathbb{R}^{|\mathcal{G}|}$ 
and $\bm{\nu} \in \mathbb{R}^{|\mathcal{H}|}$.
There may exist overlap between 
groups in $\mathcal{G}$ and groups in $\mathcal{H}$, 
and within groups in $\mathcal{G}$ or $\mathcal{H}$.
Note that $\bB$ will have zero patterns which are the union 
of groups in $\mathcal{G}$ and $\mathcal{H}$ and 
individual coefficients. 
The supports of $\bB$ (nonzero $\beta_k^j$'s)
will be the complement of zero patterns.
As the contribution of this paper is to propose an efficient optimization 
method,  for simplicity,
we assume that all weights are set to 1.
{\begin{example}
\label{subsec:example_struct_io}
We illustrate an example of the penalty used for problem (\ref{equ:reg5}).
Suppose we have 
two inputs and outputs,
$\{{\bx}_1, {\bx}_2\}$, 
$\{{\by}_1, {\by}_2\}$,
and 
$\bB$ which includes
$\{\beta_1^1, \beta_1^2,\beta_2^1,\beta_2^2\}$.
For the input and output groups, we have $\mathcal{G} = \{\bg_1\}$,
$\bg_1 = \{1,2\}$, $\mathcal{H} = \{\bh_1\}$ and
$\bh_1 = \{1,2\}$.
Under this setting, the penalty for
problem (\ref{equ:reg5}) 
is given by

\label{equ:example}
\begin{align}
\Omega(\bB) =  \lambda_1 \sum_{k=1}^2\sum_{j=1}^2 |\beta_k^j|
+ \lambda_2 \sum_{k=1}^2\sqrt{\sum_{j=1}^2(\beta_k^j)^2}
+ \lambda_3 \sum_{j=1}^2\sqrt{\sum_{k=1}^2(\beta_k^j)^2}.
\end{align}
\end{example}}

\section{Hierarchical Group-Thresholding}
\label{sec:hgroup_thres}

In this section, we propose an efficient method to optimize problem (\ref{equ:reg5})
referred to as Hierarchical Group-Thresholding (HiGT). 
Our algorithm consists of two steps.
First, We identify zero groups by checking optimality conditions (called thresholding) 
as we walk through a predefined hierarchical tree. 
After walking though the nodes in the tree, some groups of coefficients might not achieve zero. 
Second, we optimize problem (\ref{equ:reg5}) with only these groups of non-zero $\beta_k^j$'s 
using an efficient  optimization technique available for overlapping group lasso.

Let us 
characterize the zero patterns  induced by 
$\ell_1/\ell_2$ norms in problem (\ref{equ:reg5}). 
We first consider a block of $\bB_{\bh_o}^{\bg_m}$ which consists of 
one input group $(\bg_m \in \mathcal{G})$ and one output group $(\bh_o \in \mathcal{H})$.
Since each group can be zero simultaneously
($\bm{\beta}_k^{\bg_m}=\bm{0}$, 
$\bm{\beta}_{\bh_o}^{j}=\bm{0}$),
there exist zero patterns for $\bB_{{\bh_o}}^{{\bg_m}} = \bm{0}$
when $\bm{\beta}_k^{\bg_m} = \bm{0}, \; \forall k\in \bh_o$
or $\bm{\beta}_{\bh_o}^{j} = \bm{0}, \; \forall j \in \bg_m$. 
Furthermore, the union of multiple $\bB_{\bh_o}^{\bg_m}$'s
can generate zero patterns for $\bB_{\bH}^{\bG} = \bf{0}$ 
which consists of multiple input groups and multiple output groups, 
$\{\bh_o\} \in \bH$ and $\{\bg_m\} \in \bG$.
One might be able to check 
these zero patterns 
by checking optimality conditions for each $\bm{\beta}_k^{\bg_m}=\bf{0}$ and $\bm{\beta}_{\bh_o}^j=\bf{0}$.
However, this approach may be inefficient as it 
needs to examine a large number of groups. 
Instead, to efficiently check the zero patterns,
we will test multiple groups simultaneously (i.e., all groups in $\bB_{\bh_o}^{\bg_m}$ or $\bB_{\bH}^{\bG}$).
Also, we will construct a hierarchical tree, and exploit the inclusion relations
between the zero patterns
so that we can identify zero groups efficiently 
by traversing the tree while avoiding unnecessary optimality checks.

\begin{figure}[t] 
\psfrag{h1}{{\tiny $\bh_1$}}
\psfrag{h2}{{\tiny $\bh_2$}}
\psfrag{g1}{{\tiny $\bg_1$}}
\psfrag{g2}{{\tiny $\bg_2$}}
\centering
\includegraphics[width=0.45\textwidth]{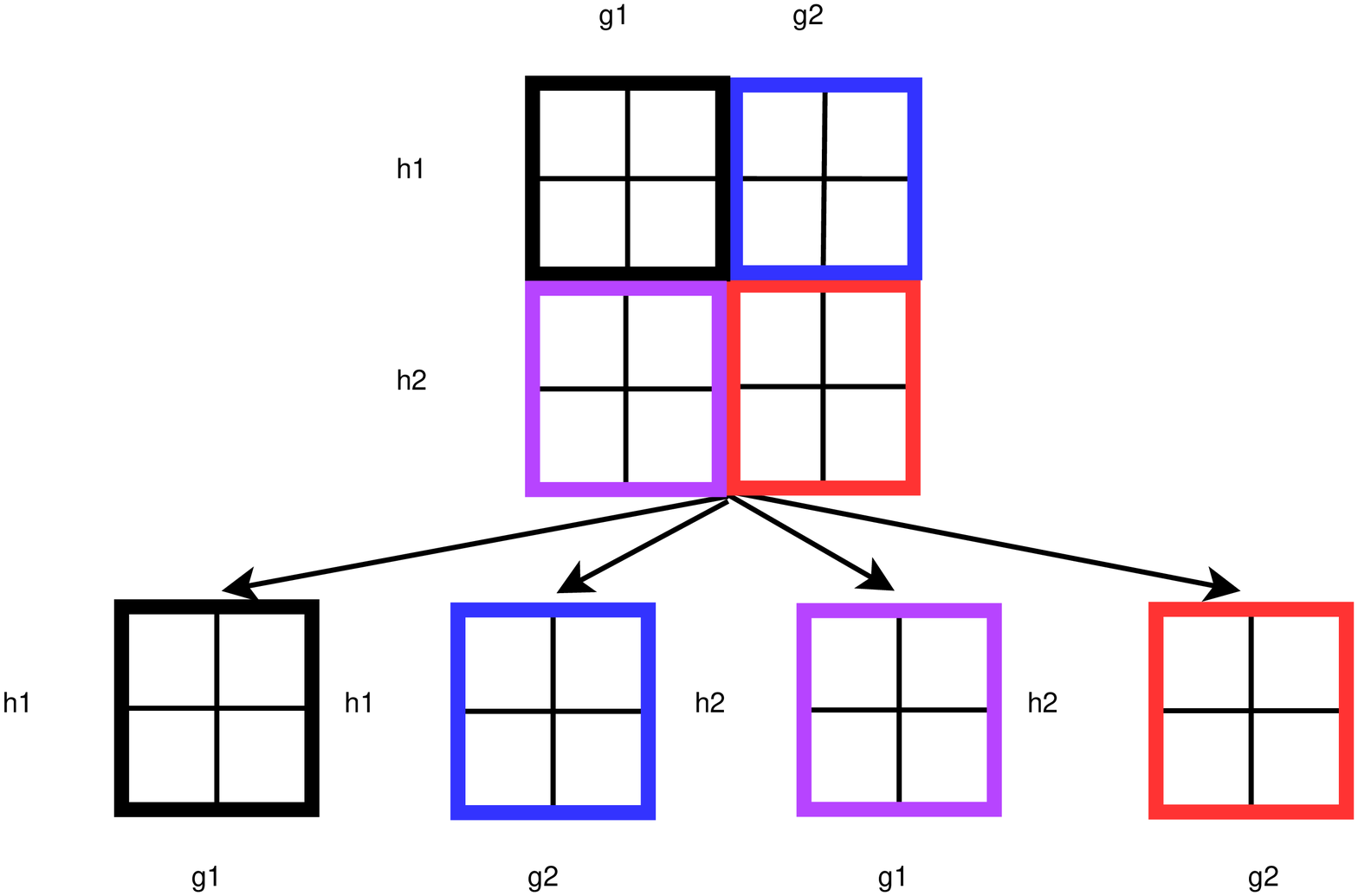}
\caption{An example of a tree that 
contains 
$\bB_{\bH}^{\bG}$,
where $\bH = \{\bh_1, \bh_2\}$, and $\bG = \{\bg_1, \bg_2\}$.
The root node contains zero pattern for $\bB_{\bH}^{\bG}=\bm{0}$,
and the leaf nodes represent the zero patterns for ${\bB}_{\bh_o}^{\bg_m}=\bm{0}$.
}
\label{fig:groupThres}
\end{figure}

In Figure \ref{fig:groupThres}, we show an example of the tree for
$\bB_{\bH}^{\bG}$ when $\bH = \{\bh_1, \bh_2\}$, $\bG = \{\bg_1, \bg_2\}$, and
$|\bg_1| = |\bg_2| = |\bh_1| = |\bh_2| = 2$.  
We denote the set of zero patterns of $\bB$ (i.e.,
$\bB_{\bh_o}^{\bg_m}$'s or $\bB_{\bH}^{\bG}$'s) 
by $\mathcal{Z} = \{Z_1,\ldots,Z_{|\mathcal{Z}|}\}$.
For example, $Z_1$ can be a zero pattern for 
$\bB_{\bH}^{\bG} = \bm{0}$ (the root node in Figure \ref{fig:groupThres}).
Let us denote $\bB(Z_t)$ by the coefficients of $\bB$ corresponding to $Z_t$'s zero pattern.
Then we define a tree as follows.
A node is represented by $Z \in \mathcal{Z}$, 
and there exists a directed edge from 
$Z_1 \in \mathcal{Z}$ to $Z_2 \in \mathcal{Z}$ if and only if $Z_1 \supset Z_2$
and $\nexists Z \in \mathcal{Z}: Z_1 \supset Z \supset Z_2$.
Note that each layer encodes different granularities of
sparsity pattern. 
When we have multiple $\bB_{{\bH}}^{{\bG}}$'s, 
we can generate a subtree for each $\bB_{{\bH}}^{{\bG}}$ separately, 
and then connect all the subtrees to the dummy root node for $\bB = \bm{0}$.

We can observe that our procedure has the following properties.
First, by testing zero conditions for each node, we can identify multiple zero groups 
simultaneously. 
Second, walking through the tree, if $\bB(Z_t) = \bm{0}$, we know that 
all the descendants of $Z_t$ are also zero due to
the inclusion relations of the tree. Hence, we can skip
to check the optimality conditions that the descendants of 
$Z_t$ are zero.

Considering these properties,
we develop our optimization method for the following reasons.
First, if $\bB$ is sparse,
our method is very efficient since 
we can skip optimality checks for many zero patterns in $\mathcal{Z}$. 
Mostly we will check only nodes
located at the high levels of the tree.
Second, our method is simple to implement. All we need
is to check whether each node in the tree attains zero. 
After identifying zero groups, we solve problem (\ref{equ:reg5}) with a small number of non-zero groups of 
coefficients using an available optimization technique. 

Specifically, our hierarchical group-thresholding method has the following procedure:

\begin{enumerate}
\item Construct a tree that contains the groups of zero patterns of $\bB$ 
(i.e., $\bB_{\bh_o}^{\bg_m}$ and $\bB_{\bH}^{\bG}$). 
In our experiments, we used two input and two output groups for each $\bB_{\bH}^{\bG}$, i.e.,
$|\bH|=|\bG|=2$.
\item 
\label{algo:iter1}
Use depth-first-search (DFS) to traverse the tree, and check optimality conditions
to see if the zero patterns at each node $Z$ achieve zero.
If $Z$ satisfies the optimality condition to be zero, 
skip the descendants of $Z$, and visit the next node according to the DFS order.
\item 
\label{algo:iter2}
With the groups of $\beta_k^j$s which did not achieve zero in the previous step, 
we solve problem (\ref{equ:reg5}) 
using an available optimization algorithm for overlapping group lasso. 
We used FoGLasso \cite{yuanefficient} for this step. 
\end{enumerate}
In the next section, we show two main
ingredients of our optimization method that include
1) the construction of a  hierarchical  tree, and
2) the optimality condition of each $Z \in \mathcal{Z}$ in the tree.

\begin{algorithm}[ht]
\caption{Hierarchical Group Thresholding (HiGT) algorithm} 
{\footnotesize 
\begin{algorithmic}
\label{alg:hGroupThres}
\STATE $\mathcal{G} \leftarrow \mbox{groups of inputs}; \mathcal{H} \leftarrow \mbox{groups of outputs}$
\STATE $T(\mathcal{Z},\mathcal{E}) \leftarrow \mbox{a hierarchical tree with groups of zero patterns 
(see Section \ref{subsec:DAG}})$
\STATE $\{Z_{(1)},Z_{(2)},\ldots,Z_{(|\mathcal{Z}|)}\}  \leftarrow \mbox{DFS order of $\mathcal{Z}$ in $T(\mathcal{Z},\mathcal{E})$}$
\STATE
\STATE {\bf(1. Screening Step)}
\STATE $V \leftarrow \emptyset$
\STATE $t \leftarrow 1$
\WHILE{$t\leq |\mathcal{Z}|$}  
\IF{$Z_{(t)}$ corresponds to $\bB_{\bH}^{\bG}=\bm{0}$}  
\STATE p $\leftarrow$ Rule in (\ref{eq:bscreen})
\ELSIF{$Z_{(t)}$  corresponds to $\bm{\bB}_{\bh_o}^{\bg_m}=\bm{0}$} 
\STATE p $\leftarrow$ Rule in Proposition \ref{equ:prop1}
\ELSE 
\STATE $t \leftarrow t+1$; continue; (Skip dummy root node)
\ENDIF
\IF{p holds (condition for $\bB(Z_{(t)}) = \bm{0}$)} 
\STATE $t \leftarrow $ DFS order of $t'$ such that 
$Z_{(t')}$ is not a descendant of $Z_{(t)}$, $t'>t$
and $\nexists t{''}: t'>t{''}>t$\\
(Skip the descendants of $Z_{(t)}$)
\ELSIF{p $=$ Rule in Proposition \ref{equ:prop1}}
\STATE $V \leftarrow V \cup $ groups in $Z_{(t)}$
(Keep the groups in $Z_{(t)}$) 
\STATE $t \leftarrow t+1$
\ELSE
\STATE $t \leftarrow t+1$
\ENDIF
\ENDWHILE
\STATE
\STATE {\bf (2. Updating Step)}
\STATE With the coefficients in $V$ and
their  corresponding  groupings in $\mathcal{G}$ and $\mathcal{H}$,
we optimize problem (\ref{equ:reg5}) using an efficient optimization  technique
for overlapping group lasso (We used FoGLasso \cite{yuanefficient} for this step).
\end{algorithmic}
}
\end{algorithm}

\subsection{Construction of Hierarchical Tree}
\label{subsec:DAG}

Here we consider each $\bB_{\bH}^{\bG}$ separately.
We first generate a tree for each  $\bB_{\bH}^{\bG}$, and then combine them to 
make a single tree. 
In each block of $\bB_{\bH}^{\bG}$, 
we examine the zero patterns of $\bB_{{\bh}_{o}}^{{\bg}_{m}}$, 
which are included in $\bB_{\bH}^{\bG}$, $\{\bg_m\}\in \bG, \{\bh_o\} \in \bH$.
These zero patterns of $\bB_{{\bh}_{o}}^{{\bg}_{m}}$ are shown in the leaf nodes in Figure \ref{fig:groupThres}. 
Note that even though we present a two-level tree throughout this paper, one can 
design a tree with multiple levels. 
Then we need to determine the edges of the tree by 
investigating the relations of the nodes.
Given  their relations between $\bB_{\bH}^{\bG}$ and $\bB_{{\bh}_{o}}^{{\bg}_{m}}$ 
(i.e., $\bB_{\bH}^{\bG} \supset \bB_{{\bh}_{o}}^{{\bg}_{m}}$), 
we create a directed
edge $Z_1 \rightarrow Z_2$. 
Finally, we make a dummy root node and generate an edge from the
dummy node to the roots of all subtrees for $\bB_{\bH}^{\bG} = \bf{0}$.

\subsection{Screening Rules for Multiple Groups} 
\label{subsec:detail_opt}

We present rules for checking zero conditions of each node in the tree.
We start with optimality condition for problem (\ref{equ:reg5}) by computing 
a subgradient of its objective function with respect to $\beta_k^j$ and set it to zero:
\begin{align}
\label{equ:kkt}
\left(\mathbf{y}_k- \bm{\beta}_k \bX\right) (\mathbf{x}_j)^T = 
\lambda_1 s_k^j + \lambda_2 c_k^j + \lambda_3 d_k^j,
\end{align}
where $s_k^j$, $c_k^j$ and $d_k^j$ are 
a subgradient of penalties in problem (\ref{equ:reg5})
with respect to $\beta_k^j$.

We first show a rule for identifying ${\bB}_{\bh_o}^{\bg_m} = \bf{0}$
which includes $|\bh_o|$ output groups and 
$|\bg_m|$ input groups of coefficients.
We assume that our algorithm starts with $\bB = \bf{0}$, and 
set $\bm{\beta}_k \bX = \bf{0}$, $\forall k$. 
Under the assumption, we can test ${\bB}_{\bh_o}^{\bg_m}=\bf{0}$ separately using Eq. (\ref{equ:kkt}).

\begin{proposition}
\label{equ:prop1}
${\bB}_{\bh_o}^{\bg_m} = \bm{0}$ if 
$\sum_{k \in \bh_o} \sum_{j \in \bg_m }  
\left| {\by}_k (\mathbf{x}_j)^T  - \lambda_1 s_k^j \right|
\leq \left|\lambda_2 \sqrt{|\bh_o|} - \lambda_3 \sqrt{|\bg_m|}\right|$ where 
\[s_k^j = \left\{ 
\begin{array}{l l}
  \frac{{\by}_k (\mathbf{x}_j)^T  }{\lambda_1} 
  & \mbox{if $\left|{\by}_k (\mathbf{x}_j)^T \right| \leq \lambda_1$}\\ 
	sign\left( \by_k (\mathbf{x}_j)^T \right) 
  & \mbox{if $\left| \by_k (\mathbf{x}_j)^T \right| > \lambda_1$}.\\  
  \end{array} \right.\]
\end{proposition}
\begin{proof}
From the optimality condition in (\ref{equ:kkt}), ${\bB}_{\bh_o}^{\bg_m} = \bm{0}$ if
\begin{align}
\sum_{k \in \bh_o } \sum_{j \in \bg_m } 
\left\{{\by}_k (\mathbf{x}_j)^T - \lambda_1 s_k^j    \right\}^2 
&= \sum_{k \in \bh_o } \sum_{j \in \bg_m }
\left\{\lambda_2 (c_k^j)+ \lambda_3  (d_k^j)\right\}^2 \nonumber\\
&\leq \lambda_2^2 |\bh_o| + \lambda_3^2 |\bg_m|
+ 2\lambda_2 \lambda_3 \sum_{k \in \bh_o } \sum_{j \in \bg_m } (c_k^j) (d_k^j)  \nonumber 	\\
&\leq  \left(\lambda_2 \sqrt{|\bh_o|} - \lambda_3 \sqrt{|\bg_m|}\right)^2.
\nonumber
\end{align}
Here we used the fact that $\sum_{j\in \bg_m}(c_k^j)^2 \leq 1$, $\sum_{k\in \bh_o}(d_k^j)^2 \leq 1$
and $\left| \sum_{k \in \bh_o } \sum_{j \in \bg_m } (c_k^j) (d_k^j)\right|^2
\leq \sum_{k \in \bh_o } \sum_{j \in \bg_m }  (c_k^j)^2 \sum_{k \in \bh_o } \sum_{j \in \bg_o }  (d_k^j)^2 \leq |\bg_m||\bh_o|$ by Cauchy-Schwarz inequality. 
The above inequality holds since 
$-\sqrt{|\bg_m||\bh_o|} \leq \sum_{k \in \bh_o } \sum_{j \in \bg_m } (c_k^j) (d_k^j)$. 
Also, $s_k^j \in [-1, 1]$ is determined to minimize the left-hand side of the inequality,
which is equivalent to applying soft-thresholding to ${\by}_k (\mathbf{x}_j)^T$.
\qed
\end{proof}

Note that Proposition \ref{equ:prop1} 
becomes the condition to identify a zero group for overlapping group lasso
when $\lambda_3 = 0$ and $K = 1$ (Lemma 2 in \cite{yuanefficient}).  
Based on Proposition \ref{equ:prop1}, we further propose a rule for identifying
$\bB_{\bH}^{\bG} = \bf{0}, \{\bg_m\} \in \bG, \{\bh_o\} \in \bH$ as follows: 
\begin{align}
\label{eq:bscreen}
{\bB}_{\bH}^{\bG} = \bm{0} \mbox{ if }
\sum_{k \in \bh_o, \bh_o \in \bH} \sum_{j \in \bg_m, \bg_m \in \bG }  
\left| {\by}_k (\mathbf{x}_j)^T  - \lambda_1 s_k^j \right|
\leq 
\sum_{\bh_o \in \bH} \sum_{\bg_m \in \bG }
\left|\lambda_2 \sqrt{|\bh_o|} - \lambda_3 \sqrt{|\bg_m|}\right|.
\end{align}

This rule does not guarantee that optimality conditions hold for 
${\bB}_{\bh_o}^{\bg_m} = \mathbf{0}$ for all $\bh_o \in \bH$ and $\bg_m \in \bG$.
However, in all of our experiments, 
we observed no violations when 
$|\bG|=|\bH|=2$, and it was very efficient
to identify a large number of zero groups simultaneously.
Here we give some motivation for this rule. 
Let us denote 
$\sum_{k \in \bh_o} \sum_{j \in \bg_m } \left| {\by}_k (\mathbf{x}_j)^T  - \lambda_1 s_k^j \right|$ by 
$L_{om}$, and 
$\left|\lambda_2 \sqrt{|\bh_o|} - \lambda_3 \sqrt{|\bg_m|}\right|$ by $R_{om}$.
If $L_{om} \leq R_{om}$ for all $(o,m)$, this rule is satisfied, and
it correctly discards groups in ${\bB}_{\bH}^{\bG}$.
Now we claim that if $L_{om} > R_{om}$ for some $(o,m)$ 
(there exist some nonzero blocks, i.e., $\bB_{\bh_o}^{\bg_m} \neq \bm{0}$),
$\Pr(\sum_{o,m} L_{om} \leq \sum_{o,m} R_{om})$ is small, and
we are unlikely to discard nonzero blocks.
Suppose $L_{om} \sim \mathcal{N}(\gamma,\sigma)$ if ${\bB}_{\bh_o}^{\bg_m} = \bm{0}$,
and $L_{om} \sim \mathcal{N}(\tau,\sigma)$ if ${\bB}_{\bh_o}^{\bg_m} \neq \bm{0}$,
where $\sigma$ is a constant, and $0< \gamma < R_{om} \leq (1+S/Q) R_{om} << \tau$.
Here $S$  and $Q$ are the number of zero and nonzero blocks in $\bB_{\bH}^{\bG}$, respectively, 
and thus $S+Q = |\bH||\bG|$.
Then,  by Hoeffding's inequality, $\Pr(\sum_{o,m} L_{om} \leq \sum_{o,m} R_{om}) \leq 
\exp\left\{-2 \left({\mathbb E}(\sum_{o,m} L_{om}) - \sum_{o,m} R_{om}\right)^2/C 
\right\}$, where $C$ is a constant. 
We can see that if $\bB_{\bH}^{\bG} \neq \bf{0}$, 
$\Pr(\sum_{o,m} L_{om} \leq \sum_{o,m} R_{om})$ is likely to be small since 
${\mathbb E}(\sum_{o,m} L_{om}) = \tau Q + \gamma S  >> (1+S/Q)\sup \{R_{om}\} Q + \gamma S >  \sum_{o,m} R_{om}$.
Therefore, the rule in (\ref{eq:bscreen}) would work well when $S/Q$ is small since 
the assumption for $\tau$ can be weak.
However, it should be noted that if $|\bG|+|\bH|$ is large, 
$S/Q$ can be very large ($S>>Q$), and the assumption for $\tau$ becomes too strong.
As a result, this rule may be violated if we test very large blocks.

From computational perspective, the rule in 
(\ref{eq:bscreen}) significantly decreases 
the number of iterations for identifying zero groups
as we can test a block of coefficients consisting of multiple
input groups and multiple output groups.
Note that each test can be performed very efficiently  
by summation of elements in a pre-computed matrix, and
the speed for each test can potentially be further improved by GPU \cite{bolz2003sparse}.

\section{Experiments}
In this section, we show the efficiency and accuracy of our proposed method using simulated datasets,
and present its usefulness for eQTL mapping, an important application in bioinformatics. 
We also present comparison between our optimization method and three other competitors
including Fast overlapping Group Lasso (FoGLasso) \cite{yuanefficient}, 
Smoothing Proximal Gradient method (SPG) \cite{chen2010efficient}, 
and Structured Lasso algorithm (SLasso) \cite{jenatton2009structured}.
Note that FoGLasso is a state-of-the-art method for overlapping group lasso, and 
Yuan et al. showed that FoGLasso is significantly faster than other alternative methods \cite{yuanefficient}.

We designed our experiments as follows. 
In section \ref{subsec:screening}, we first
present the efficiency of screening step in our method for a wide range of 
tuning parameters. 
Then, 
we present the speed and accuracy of our method
under various settings  in comparison to other methods.
Finally, we confirm the usefulness of our method by showing 
an interesting interaction effect between a pair of genetic variants 
in yeast that we identified using our method.

\subsection{Evaluation of Efficiency of Our Method Via Simulation Study}
\label{subsec:screening}
To systematically evaluate the efficiency of our method, 
we generated simulated datasets as follows. 
For generating $\bX \in \mathbb{R}^{J \times N}$, 
we first selected $J$ input covariates from a uniform distribution over $[0,1]$ for $N$ samples.
Then we defined input and output groups as follows. 
For input groups, we selected the size of input groups 
from a uniform distribution over $[5,10]$, denoted by $\mathcal{U}(5,10)$, and
the size of overlapping inputs between two  consecutive groups was selected from 
$\mathcal{U}(1,4)$. For output groups, 
the size of output groups was selected from $\mathcal{U}(3,5)$, and 
the size of overlap with the previous output group was drawn from $\mathcal{U}(1,2)$.
We then simulated $\bB \in \mathbb{R}^{K \times J}$, i.e, the ground-truth coefficients, 
which includes 52 nonzero coefficients ($\beta_k^j = 3$).
Given $\bX$ and 
$\bB$, we generated $K$ outputs 
by $\bY=\bB \bX + \bE,$ 
$\bE \sim \mathcal{N}(\bm{0},\mathbf{I})$.
We generated 10 different datasets for each simulation setting 
with $N$, $K$ and $J$, and 
report the average CPU time and average accuracy
using F1 score, which is 
harmonic mean of precision and recall rates.
Given an estimated $\bB$, precision is defined by the ratio of the
number of correctly found nonzero coefficients to 
the total number of estimated nonzero coefficients, and 
recall is denoted by the number of correctly found nonzero coefficients 
divided by the total number of true nonzero coefficients.
Throughout all the experiments, we employed a two-level tree (excluding the dummy root node), 
where the nodes at the first level contain a block of coefficients consisting of
two input groups and two output groups, and the leaf nodes include a block of 
cofficients with one input group and one output group.

\paragraph{Evaluation of Efficiency of Screening Step Via Simulation Study}
We first evaluate the efficiency of screening step (the first step in  
Algorithm \ref{alg:hGroupThres})
for a range of tuning parameters 
$\{0.001,0.002,0.005,0.01,0.02,0.05,0.1,0.2,0.5\}$
using simulation datasets with $N=1000$, $J=5000$ and $K=5$, and
Table \ref{tab:screen} shows the results.
For simplicity, we set $\lambda_1=\lambda_2=\lambda_3$ denoted by $\lambda$.
From the table, we can observe that screening time 
drops significantly as $\lambda$ changes from $0.05$ to
$0.02$, which indicates that many coefficients were discarded in the first level 
of our hierarchical tree. Indeed, the number of selected groups was decreased from 17071 to 116
without missing true nonzero coefficients.
It should be noted that the updating time (the second step in Algorithm \ref{alg:hGroupThres}) 
was also substantially
reduced when $\lambda$ is changed from $0.02$ to $0.05$ due to the small number of 
groups selected by screening step. 
Thus, our algorithm became very efficient from  $\lambda = 0.05$ since
both screening and updating step were very fast.
For large tuning parameters (e.g. $\lambda \geq 0.2$),
we started to miss true coefficients, and when $\lambda = 0.5$, 
all coefficients were set to zero due to heavy penalization. 
We can observe that $\lambda= 0.05$ or $0.1$ are appropriate for
our simulation datasets, 
and in the following experiments, we will use these two tuning parameters.

\begin{table}
\caption{Efficiency of our screening step for a range of 
tuning parameters. For comparison, CPU time for updating step (the second step in  
Algorithm \ref{alg:hGroupThres}) is also presented. 
The fourth column denotes the number of groups selected by our screening step
(total number of groups: 19152), and 
the last column represents the number of true nonzero coefficient 
discarded by our screening step. }
{\footnotesize
\begin{center}
    \begin{tabular}{ | l | l | l | l |  l | p{5cm} |}
    \hline
    $\lambda_1=\lambda_2=\lambda_3$ & Screening Time (s) & Updating Time (s) 
    & \# Selected Groups& \# Missing $\beta_k^j\neq 0$  \\ \hline
    0.001 & 0.465  & 13.246  & 19152   & 0\\ \hline
    0.002 & 0.482  & 13.497  & 19152  & 0\\ \hline
		0.005 & 0.470  & 12.515  & 19151  & 0\\ \hline    
		0.01  & 0.481  & 9.343   & 19124  & 0\\ \hline
		0.02  & 0.476  & 6.018   & 17071  & 0\\ \hline
		0.05  & 0.246  & 0.022   & 116   & 0\\ \hline
		0.1   & 0.255  & 0.010   & 51   & 0\\ \hline
		0.2   & 0.249  & 0.003   & 16  & 20\\ \hline
		0.5   & 0.239  & 0       & 0  & 52\\ 
    \hline
    \end{tabular}
\end{center}
}
\label{tab:screen}
\end{table}

\paragraph{Evaluation of Speed and Induced Structured Sparsity Via Simulation Study} 

\begin{figure}[t] 
\centering
\subfigure[]{\includegraphics[width=0.4\textwidth]{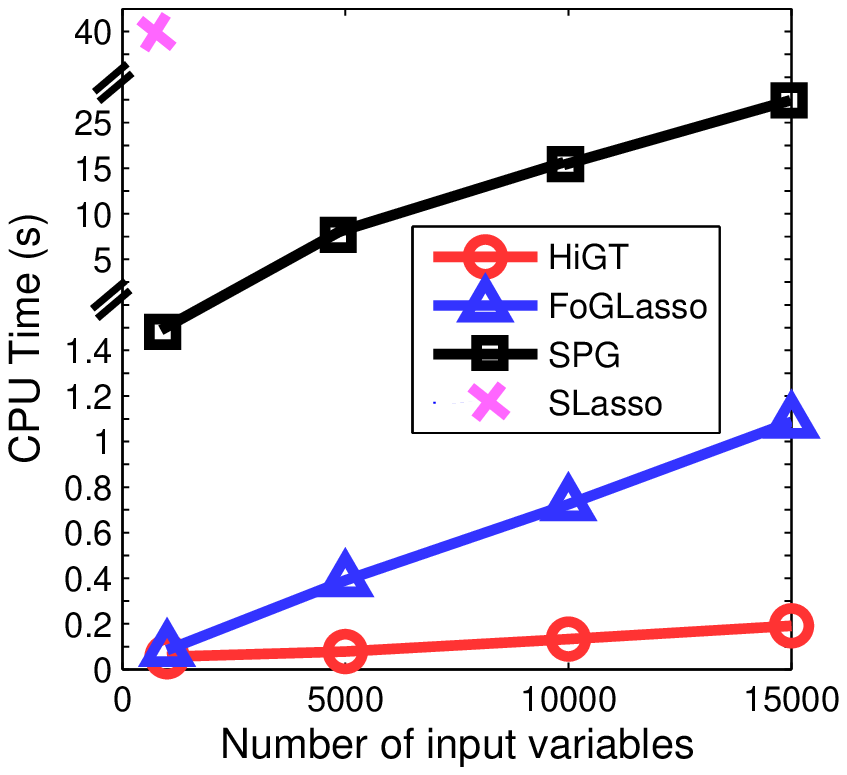}}
\subfigure[]{\includegraphics[width=0.4\textwidth]{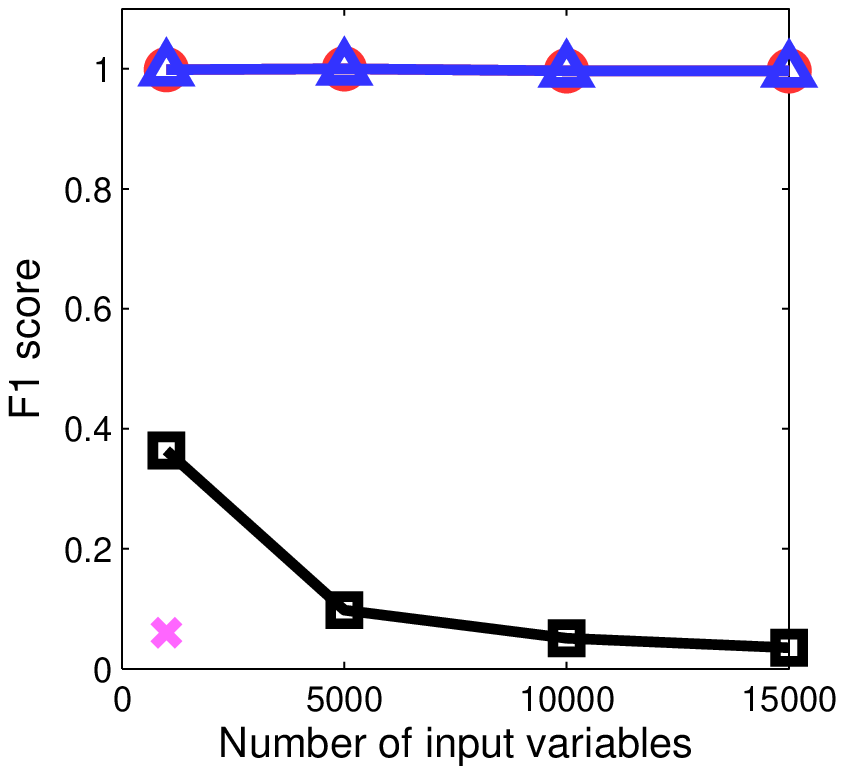}}
\caption{(a) CPU time and (b) F1 score comparison of our proposed HiGT method, FoGLasso, 
SPG, and SLasso with different number of input variables under 
a single task regression setting. 
We used simulation datasets with $N=1000$, $K=1$, $\lambda_1=\lambda_2 = 0.05$, and  
$\lambda_3 = 0$. }
\label{fig:speed_fullcomp}
\end{figure}

We compared the speed and the accuracy 	of our HiGT method with the three alternatives
of FoGLasso, SPG and SLasso. 
We first show the results under a single task regression setting where $\lambda_3 = 0$, and $K=1$
(this setting was used in previous papers for FoGLasso, SPG and SLasso).
Figure \ref{fig:speed_fullcomp}(a,b) show CPU time and F1 score of the four methods 
with different number of input variables from $1000$ to $20000$, fixing $N=1000$, $K=1,
\lambda_1 = \lambda_2= 0.05$, and $\lambda_3 = 0$. 
We observed that our method was much more scalable than other methods, 
and perfectly recovered true nonzero coefficients.
FoGLasso achieved the same accuracy but it was not as fast as HiGT
due to the lack of hierarchical group screening step.
In the following comparison analysis, we included only 
FoGLasso and SPG which showed good performance.

\begin{figure}[htp]
\centering
\subfigure[]{\includegraphics[width=0.23\textwidth]{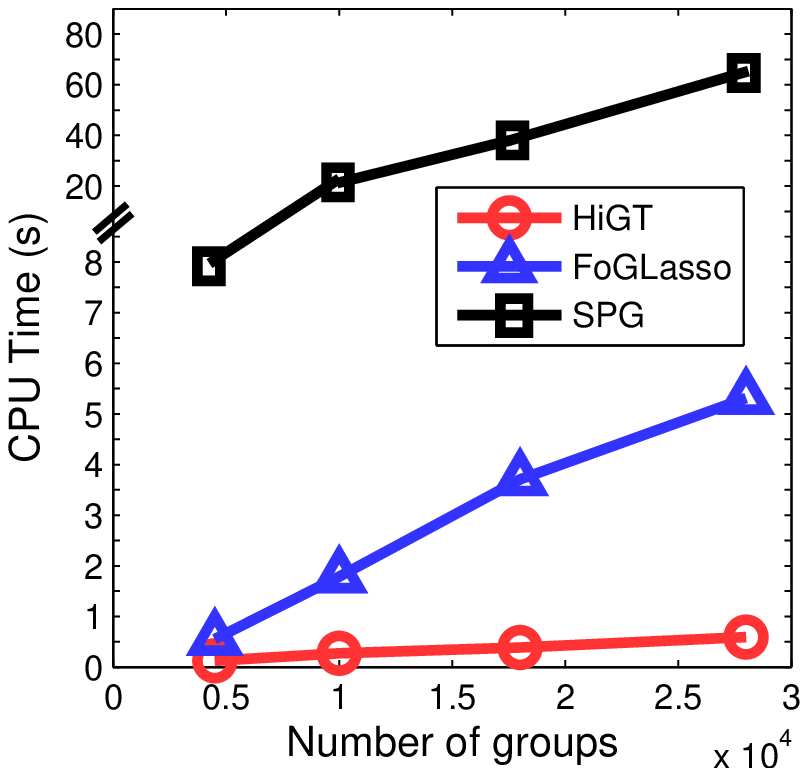}}
\subfigure[]{\includegraphics[width=0.23\textwidth]{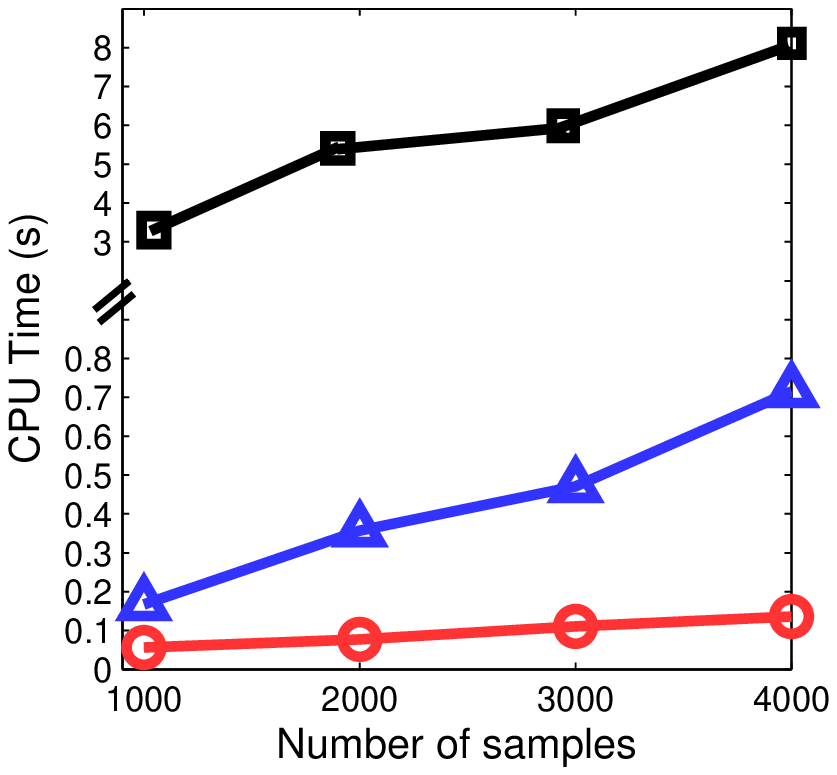}}
\subfigure[]{\includegraphics[width=0.23\textwidth]{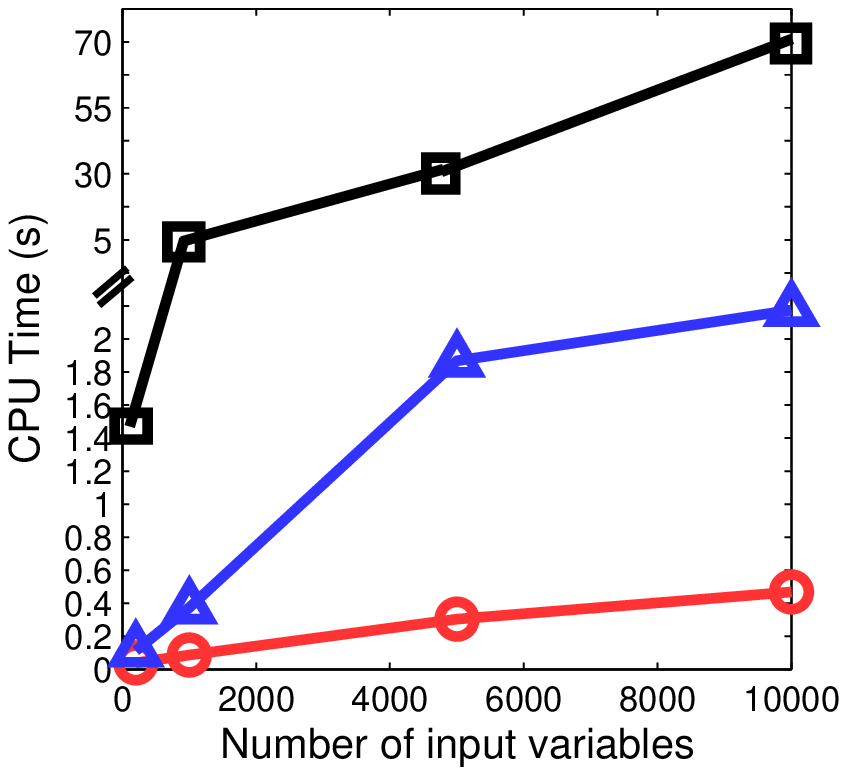}}
\subfigure[]{\includegraphics[width=0.23\textwidth]{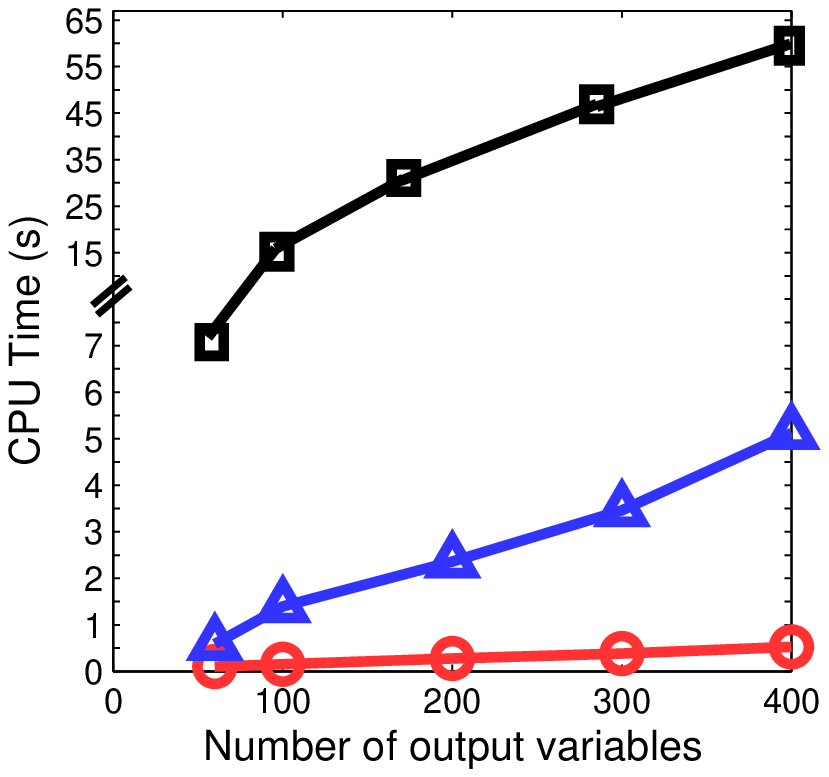}}
\subfigure[]{\includegraphics[width=0.23\textwidth]{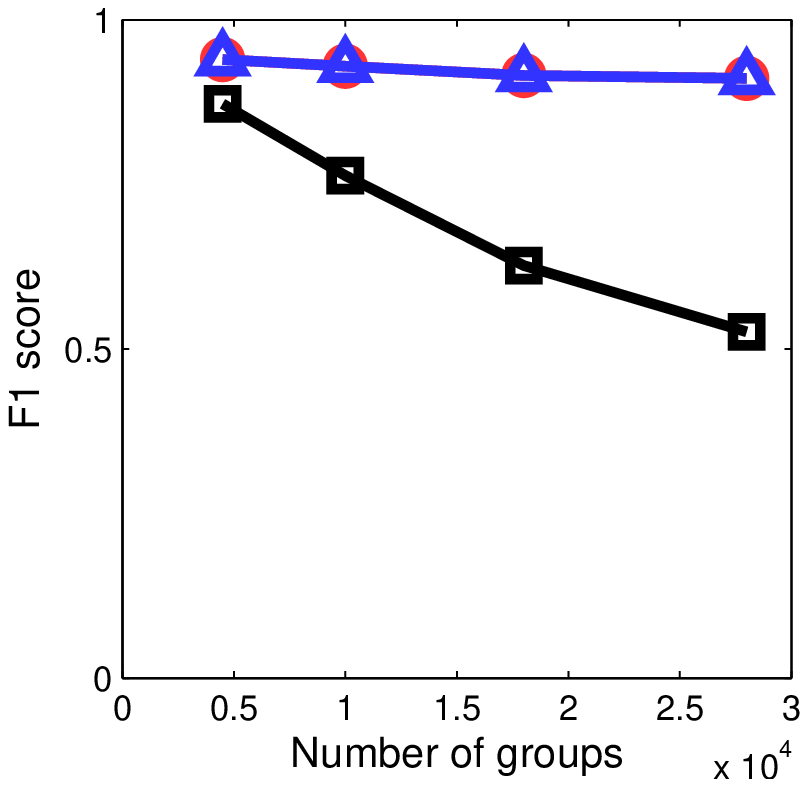}}
\subfigure[]{\includegraphics[width=0.23\textwidth]{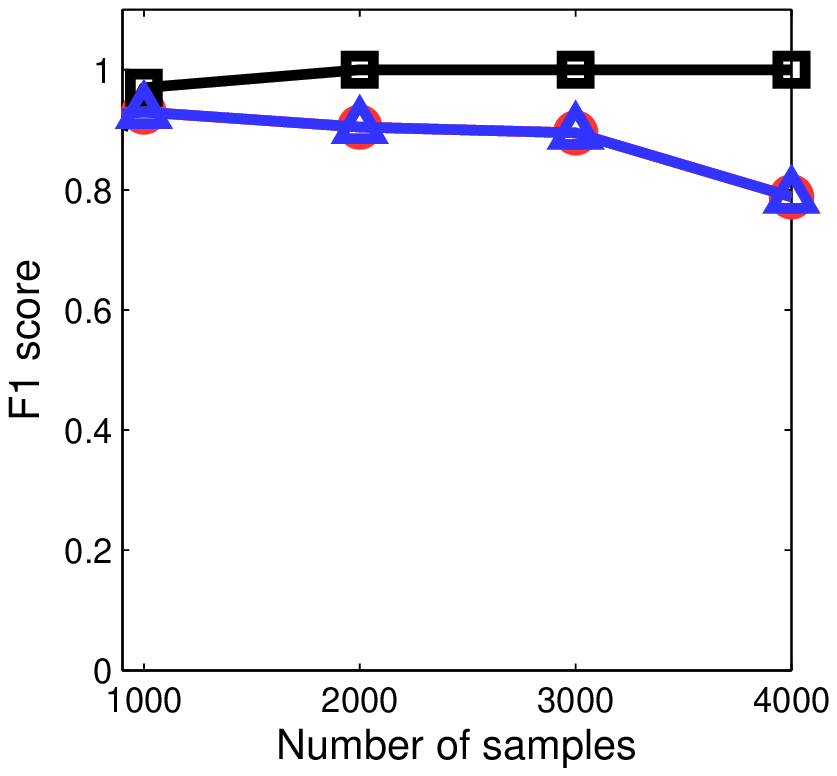}}
\subfigure[]{\includegraphics[width=0.23\textwidth]{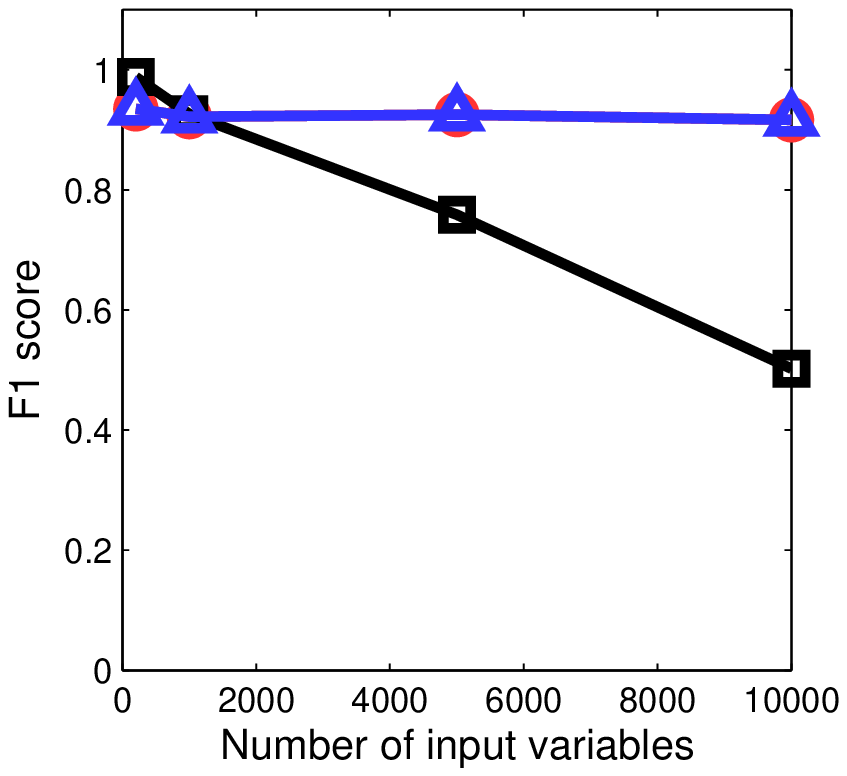}}
\subfigure[]{\includegraphics[width=0.23\textwidth]{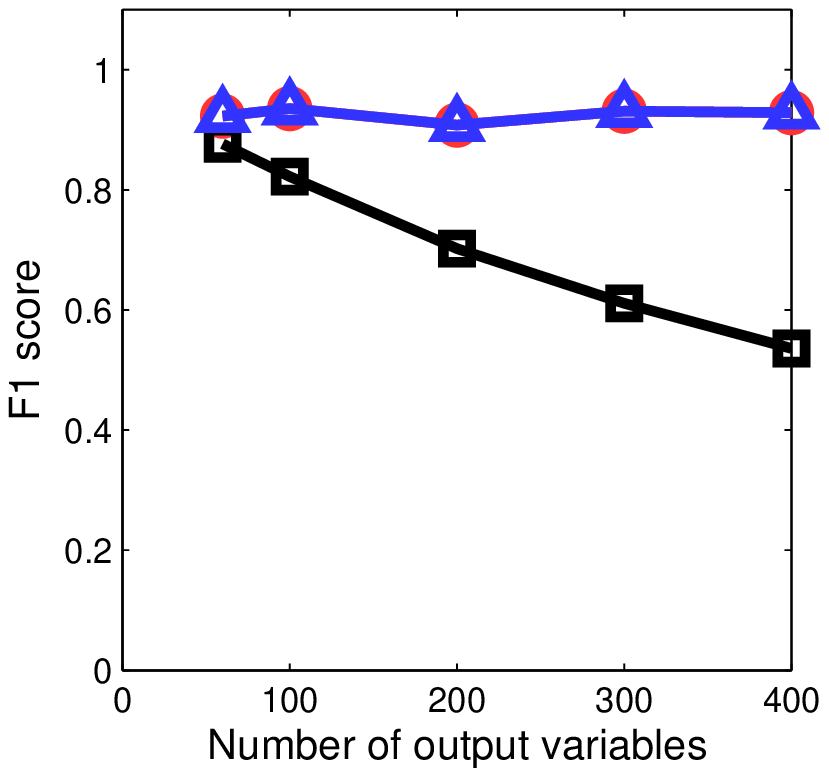}}
\caption{CPU time and  F1 score comparison of our proposed HiGT method, FoGLasso, and SPG with different 
(a,e) number of groups ($N=1000$), 
(b,f) samples ($J=500, K=5$), 
(c,g) input variables ($N=1000, K=5$), and 
(d,h) output variables ($N=1000, J=150$).}
\label{fig:performance}
\end{figure}

Figure \ref{fig:performance} shows efficiency and F1 score of three methods 
including HiGT, FoGLasso, and SPG under various simulation settings.
For all experiments, we set $\lambda_1 = \lambda_2=\lambda_3 = 0.1$.
From this figure, we can observe the following:
\begin{itemize}
	\item For all settings with 
	different number of groups, samples, input and output variables, 
	our algorithm was much more efficient than the other methods. 
	\item Our HiGT algorithm and FoGLasso showed the same F1 score (close to 1) 
	for all simulation settings. 
	\item Screening step in our algorithm never made a mistake for all experiments.
	\item In general, the accuracy of HiGT and FoGLasso did not decrease as the problem size increased.
	\item For all methods, CPU time increased linearly with 
	the number of groups but the slopes were significantly different.
	Our HiGT method has a very small slope due to the efficient screening step.
\end{itemize}

\subsection{Detecting eQTLs Having Interaction Effects in Yeast Genome}
We also solved problem (\ref{equ:reg5}) 
using our HiGT method 
with yeast data \cite{brem2005landscape}  
which contains 1260 unique SNPs 
and observed gene-expression levels of 5637 genes.
To show the usefulness of our method, 
we briefly  report the most significant eQTLs having interaction effects that we identified
(chr1:154328-chr5:350744).
According to our estimation, 
it turns out this pair of genetic variants affected 455 genes enriched with 
the GO category of ribosome biogenesis with corrected p-value $< 10^{-35}$.
This SNP pair was very closely located on gene NUP60 and gene RAD51, respectively, and
we found that 
there exists a significant genetic interaction between the two genes
\cite{costanzo2010genetic}.
As both SNPs are closely located to NUP60 and RAD51 (within 500bp), 
we can assume that the two SNPs 
affected the two genes (NUP60 and RAD51),
and their genetic interaction in turn acted on a large number of
genes related to ribosome biogenesis. It implies that 
this pair of SNPs can be a truly meaningful biological finding.
We consider that our detection of this SNP pair is novel as
the exact locations of the SNP pair were not reported in both
Storey et al. \cite{storey2005multiple} and 
a statistical test for pairwise interactions \cite{purcell2007plink}.

\section{Discussions}
\label{sec:discussion}
In this paper, we presented an efficient algorithm 
for a large-scale overlapping group lasso problem in highly general settings.
Our method relies on a screening step 
which can efficiently discard a large number of irrelevant groups simultaneously.
Our simulation confirmed that our model is significantly
faster than other competitors while maintaining high accuracy.
In our analysis of yeast eQTL datasets, we reported a pair of genetic variants 
that potentially interact with each other and influence on ribosome biogenesis.

One of promising research directions of this work would be to consider 
parallelization of our method. Note that we can naturally
parallelize the screening step as it considers a set of groups separately. 
However, the second step of our algorithm 
needs to be performed sequentially after the screening step is completed. 
A efficiently parallelized algorithm would not only
further speed up the algorithm but also allow us to 
deal with very large problems which cannot fit into memory.
We are also interested in theoretical analysis of our screening step
in terms of sure screening property for ultra high dimensional problems \cite{fan2008sure}
or the properties of strong rules for discarding covariates \cite{tibshirani2011strong}.
Finally, we plan to apply our efficient algorithm to very large-scale
eQTL mapping problems in bioinformatics for understanding the biological
mechanisms of complex human diseases.

\bibliographystyle{plain}

\end{document}